\documentclass[journal]{IEEEtran}
\usepackage{booktabs}  

\usepackage{amsmath,amsfonts,amssymb,bm}
\usepackage{amsthm}
\theoremstyle{definition}
\newtheorem{lemma}{Lemma}
\newtheorem{theorem}{Theorem}
\usepackage{graphicx}  
\usepackage{float}     
\usepackage{algorithm}
\usepackage{multirow}  
\usepackage{algpseudocode}

\usepackage{array}
\usepackage[caption=false,font=footnotesize,labelfont=rm,textfont=rm]{subfig}
\usepackage{stfloats}
\usepackage{url}
\usepackage{verbatim}
\usepackage{cite}
\usepackage[table,xcdraw]{xcolor}
\usepackage[colorlinks=true,
            linkcolor=blue,       
            anchorcolor=red,     
            citecolor=blue,       
            urlcolor=black,
            ]{hyperref}
 
\usepackage{makecell}
\begin{document}  

\title{Motion-Enhanced Nonlocal Similarity Implicit Neural Representation for \\Infrared Small Target Detection}
\author {
    Pei Liu\textsuperscript{\rm 1},
    Yisi Luo\textsuperscript{\rm 1}\textsuperscript{*},
    Wenzhen Wang\textsuperscript{\rm 1},
    Jun-Jie Zhang\textsuperscript{\rm 2},
    Hui Qiao\textsuperscript{\rm 3},
    Xiangyong Cao\textsuperscript{\rm 1}\textsuperscript{*}\\
    1 Xi'an Jiaotong University, Xi'an 710000, China,\\
    2 Northwest Institute of Nuclear Technology, Xi'an 710000, China,\\
    3 China Telecom Shaanxi Branch, Xi'an 710000, China.
}

\maketitle

\begin{abstract}
Infrared small target detection presents a significant challenge due to dynamic multi-frame scenarios and small target signatures in the infrared modality. Traditional low-rank plus sparse models often fail to capture dynamic backgrounds and global spatial-temporal correlations, which results in background leakage or target loss. In this work, we propose an unsupervised motion-enhanced nonlocal similarity implicit neural representation (INR) framework to address these challenges. Specifically, we first integrate motion estimation via optical flow to capture subtle target movements, and propose multi-frame fusion to enhance motion saliency. Second, we leverage nonlocal similarity to construct patch tensors with strong low-rank properties, and propose an innovative tensor decomposition-based INR model to represent the nonlocal patch tensor, effectively encoding both the nonlocal low-rankness and spatial-temporal correlations of background through continuous neural representations. An alternating direction method of multipliers (ADMM) is developed for the nonlocal INR model, which enjoys theoretical fixed-point convergence. Experimental results show that our approach robustly separates small targets from complex infrared backgrounds, outperforming state-of-the-art methods in detection accuracy and robustness.

\end{abstract}


\section{Introduction}
\IEEEPARstart{I}{nfrared} small target detection (IRSTD) is widely used in target warning, maritime rescue, and long-range search \cite{wu2022srcanet, yan2023stdmanet}, etc. The small targets, typically less than 0.15\% of the entire image, exhibit weak features, while complex backgrounds introduce clutter, making IRSTD a challenging task. Specifically, multi-frame IRSTD primarily encounters two major challenges: (1) \textit{\textbf{Dynamic Target Capture.}} The dynamic features of small targets across frames limit the capability of single-frame methods. In contrast, existing multi-frame methods exploit spatial-temporal differences between the target and background to achieve better performance in detecting moving targets. However, these methods still encounter target loss between frames, leaving scope for accuracy improvement. (2) \textit{\textbf{Accurate Background Estimation.}} Effective background modeling is crucial for target separation. Among the two main approaches for IRSTD, existing model-based methods~\cite{gao2013infrared, luo2022imnn, ma2023weighted, liu2024infrared, luo2024clustering} utilize domain knowledge to design effective regularizers to address this issue. These methods are interpretable and independent of data volume and labels but have limited representational capacity for complex scenes. Data-driven methods~\cite{dai2021attentional, zhang2024irprunedet, chen2024tci, zhang2025irmamba, yang2025pinwheel}, which use complex neural networks, provide stronger background representations but depend on labeled data and may struggle with out-of-distribution data.\par
\begin{figure}[t] 
  \centering
  \includegraphics[width=0.98\linewidth,height=0.41\linewidth]{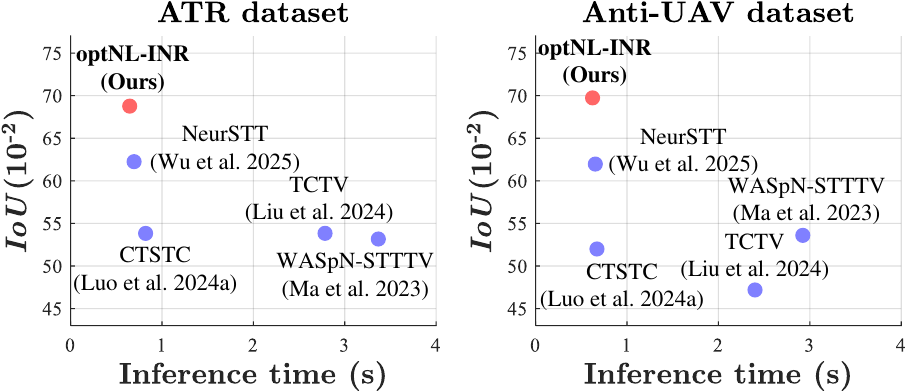}
  \vspace{-0cm} 
  \caption{Average detection performance ($IoU$) vs. inference time (s) of several unsupervised detectors for IRSTD.}
  \vspace{-0cm}
\end{figure}
To address these challenges, combining prior-based optimization models with deep networks for IRSTD is a promising research. The development of deep unfolding networks~\cite{RPCANet}, which unroll prior-guided models within the network, has enhanced network interpretability but still relies on labeled data. In contrast, unsupervised learning methods using deep networks for feature representation overcome this dependency. \cite{zhang3DSTPM} introduced a dataset-free deep prior expressed by the untrained 3-D spatial-temporal prior module (3DSTPM) for representation learning, which consumes massive parameters from 3DCNN. In contrast, implicit neural representations (INRs), as a powerful paradigm for unsupervised tensor data representation, use a lightweight multilayer perceptron (MLP) to map coordinate tensors to continuous space, modeling complex data and achieving strong performance in various vision tasks \cite{NeRF, SIREN, LRTFR}. For instance, \cite{LRTFR} introduced a low-rank constrained INR, which incorporating prior knowledge to regularize the solution and improve background recovery.\par 
Inspired by the pioneer works on INRs, we propose a nonlocal low-rank INR method for IRSTD, incorporating nonlocal similarity and dynamic multi-frame background representation through motion estimation to improve target separation. Specifically, the proposed motion-enhanced nonlocal similarity INR model (optNL-INR) combines both the strengths of model-based and data-driven approaches by leveraging sufficient interpretable domain knowledge (nonlocal low-rankness and motion information) and the expressiveness of INRs. First, using motion information between sequential frames, we compute the motion intensity of targets via optical flow, improving target location estimation and reducing target loss. Second, we construct a spatial-temporal tensor (STT) model by acquiring nonlocal similar patches, and approximate background patches using a novel nonlocal low-rank INR within the continuous domain, capturing nonlocal low-rankness and global spatial-temporal correlations of the multi-frame background. This enables more accurate background estimation and target separation. Finally, we use 3D total variation (3DTV) and alternating direction method of multipliers (ADMM) to ensure stable optimization. The proposed motion-enhanced nonlocal INR method for IRSTD is unsupervised and does not rely on labeled data.\par
In summary, the contributions of our work are four-fold:
\begin{itemize}
\item We propose a novel unsupervised optNL-INR method for IRSTD. The optNL-INR leverages a dynamic multi-frame optical flow fusion strategy to estimate subtle target motion, enabling robust motion saliency estimation that strengthens target localization.
\item We propose a nonlocal INR model with tensor Tucker decomposition to represent infrared backgrounds, which effectively captures both the nonlocal low-rankness and global spatial-temporal correlations of infrared backgrounds, thus improving small target detection.
\item We provide rigorous theoretical guarantees for the existence of the nonlocal INR model, the spatial-temporal correlation bound from the INR smoothness, and fixed-point convergence of the ADMM, ensuring solid theoretical foundations and reliability of our method.
\item Extensive experiments on infrared multi-frame datasets demonstrate the superior performance of our motion enhanced nonlocal INR method in detecting small targets in infrared images against state-of-the-art baselines.
\end{itemize}

\section{Related Works}
\subsection{Supervised Methods}
In recent years, supervised deep learning (DL) methods for IRSTD have achieved significant performance improvements. The early-developed attention networks~\cite{2021Attentional,wang2022interior} effectively alleviated small target feature loss in deep networks. Later, several UNet-based architectures~\cite{li2023direction, zhong2024csan} enhanced multi-scale feature extraction, significantly enhancing both global and local contrast information, thereby improving small target detection capability. Furthermore, frequency domain-based methods~\cite{zhu2025towards, liu2025spatial} exploited spatial-frequency differences between the target and background, effectively suppressing background clutter while extracting target features. With the development of visual Transformers, methods such as RKformer~\cite{zhang2022rkformer}, PBT~\cite{yang2024pbt}, and SCTransNet~\cite{yuan2024sctransnet} leveraged the global attention mechanism of Transformers to capture contextual relationships between small targets and complex backgrounds, overcoming CNN's local receptive field limitations. Additionally, advances in deep unfolding networks have achieved remarkable progress for IRSTD by unfolding a prior-guided optimization model into a network. \cite{RPCANet} introduced a theory-guided neural network based on the RPCA model for IRSTD. \cite{Deep-LSP-Net} proposed Deep-LSP-Net, a patch-based network that decomposes infrared images into low-rank backgrounds and sparse targets through patch-based processing, and \cite{DUSRNet} developed an end-to-end framework integrating sparse regularization with adaptive background estimation and target extraction modules. Although these data-driven methods improve performance, they rely on supervised training with large labeled datasets and suffer from limited interpretability due to their black-box structure. As compared, the proposed unsupervised optNL-INR would be more generalizable across different scenes and enjoys better interpretability. \par

\subsection{Unsupervised Methods}
Traditional unsupervised methods initially evolved from classical background filtering methods~\cite{rivest1996detection}. Subsequently, human visual system \cite{han2020infrared,lu2023infrared} and low-rank sparse decomposition (LRSD) were introduced. The infrared patch image (IPI) model \cite{gao2013infrared} based on LRSD has inspired a series of methods. \cite{dai2017reweighted} introduced the reweighted IPI model, and \cite{kong2021infrared} proposed nonconvex tensor fibered rank approximation for IRSTD. Due to the limited availability of spatial information, spatial-temporal tensor (STT) models have been proposed for multi-frame IRSTD. \cite{luo2022imnn} proposed a multi-mode nuclear norm joint local weighted entropy contrast (IMNN-LWEC) method via optimization-based decomposition of the spatial-temporal tensor. \cite{ma2023weighted} proposed a weighted adaptive Schatten p-norm and spatial-temporal tensor transpose variation (WASpN-STTTV) model for maritime IRSTD, boosting target detection performance in non-uniform sea wave backgrounds. \cite{liu2023infrared} proposed a nonconvex tensor Tucker decomposition model with factor prior for IRSTD, addressing limitations of predefined rank selection and enhancing detection in complex scenes. \cite{luo2024clustering} proposed a clustering and tracking-guided spatial-temporal prediction completion model (CTSTC) in the high-frequency domain, integrating an improved k-means algorithm and Bayesian tracking regularization to address target-background separation and real-time detection in complex infrared scenes. Additionally, implicit neural representations (INRs) have been introduced for unsupervised tensor data representation. \cite{wu2025neural} proposed a neural STT model based on INR for IRSTD. These unsupervised methods combine theoretical interpretability with practical performance. Compared to these methods, we propose a novel motion-enhanced nonlocal INR model, which incorporates nonlocal similarity and motion information, leveraging INR’s expressiveness for superior detection accuracy.

\section{Notations and Preliminaries}
\noindent {\bf Notations} The scalar, vector, matrix, and tensor are denoted as $x$, ${\bf x}$, ${\bf X}$, and $\mathcal{X}$. Given a tensor $\mathcal{X}\in \mathbb{R}^{n_1 \times n_2 \times n_3}$, the element at position $(i, j, k)$ in $\mathcal{X}$ is denoted as $\mathcal{X}(i, j, k)$. The Frobenius norm of a tensor $\mathcal{X} \in \mathbb{R}^{n_1 \times n_2 \times n_3}$ is defined as $\|\mathcal{X}\|_{F} = \sqrt{\sum_{i, j, k} \mathcal{X}(i, j, k)^2}$.
The mode-$i$ ($i=1,2,3$) unfolding operator of a tensor $\mathcal{X} \in \mathbb{R}^{n_1 \times n_2 \times n_3}$ results in a matrix denoted by $
{\tt unfold}_{i}(\mathcal{X}) = {\bf X}_{(i)} \in \mathbb{R}^{n_{i} \times \prod_{j \neq i} n_{j}}$. The operator ${\tt fold}_{i}(\cdot)$ is defined as the inverse of ${\tt unfold}_{i}(\cdot)$. The mode-$i$ product between a tensor $\mathcal{X}\in \mathbb{R}^{n_1 \times n_2 \times n_3}$ with a matrix ${\bf A}\in \mathbb{R}^{n\times n_i}$ is defined as $
\mathcal{X} \times_{i} {\bf A} = {\tt fold}_{i}\left({\bf A} {\bf X}_{(i)}\right)$. 
The Tucker rank of a tensor
\begin{math}
  \mathcal{X} \in \mathbb{R}^{n_1 \times n_2 \times\cdots\times n_N}
\end{math} is defined as $\mathrm{rank}_{T}(\mathcal{X}) = \left(\mathrm{rank}({\bf X}_{(1)}), \mathrm{rank}({\bf X}_{(2)}), \cdots,\mathrm{rank}({\bf X}_{(N)})\right).$
\begin{lemma}[Tensor Tucker decomposition
\cite{Tucker}]\label{Tucker}
For a tensor $\mathcal{X} \in \mathbb{R}^{n_1 \times n_2 \times\cdots\times n_N}$ with Tucker rank $\mathrm{rank}_{T}(\mathcal{X}) = (r_1,\cdots,r_N)$, there exist a core tensor $\mathcal{C} \in \mathbb{R}^{r_1 \times\cdots \times r_N}$ and $N$ factor matrices ${\bf U}_1 \in \mathbb{R}^{n_1 \times r_1},\cdots,{\bf U}_N \in \mathbb{R}^{n_N \times r_N}$ such that $\mathcal{X}$ can be represented by $\mathcal{X} = \mathcal{C}\times_{1} {\bf U}_1\times_2\cdots\times_{N} {\bf U}_N .$
\end{lemma}

\section{Methodology of optNL-INR}
\subsection{Motion Enhancement for Infrared Images}
Optimization-based IRSTD approaches leverage the intrinsic low-rank structure of backgrounds and the sparsity of targets under the robust principal component analysis (RPCA) framework \cite{dai2017reweighted} to represent STT, which is formulated as:
\begin{equation}\label{Init_RPCA_model}
\begin{aligned}
    \mathcal{D} =\mathcal{B} +\mathcal{T} +\mathcal{N},
\end{aligned}    
\end{equation}
where $\mathcal{D} \in \mathbb{R} ^{n_1\times n_2\times n_3}$ is constructed by the input infrared image sequences ${\bf I}_1, {\bf I}_2, ..., {\bf I}_{n_3}\in \mathbb{R} ^{n_1\times n_2}$. $\mathcal{B} , \mathcal{T} , \mathcal{N} \in \mathbb{R} ^{n_1\times n_2\times n_3}$ are the background, target, and noise tensors, respectively. The overall flowchart of our proposed optNL-INR model for IRSTD is shown in Figure \ref{fig:Frame}. \par 
To effectively utilize motion information, we leverage the infrared small target characteristic and the motion intensity of pixel points in continuous scenes to capture subtle target motion in $\mathcal{D}$. We utilize the Farneback optical flow \cite{farneback2003two} to determine the motion intensity and direction of moving targets, effectively extracting motion information in multi-frame images.
\begin{figure}[t]
  \centering
  \includegraphics[width=0.95\linewidth]{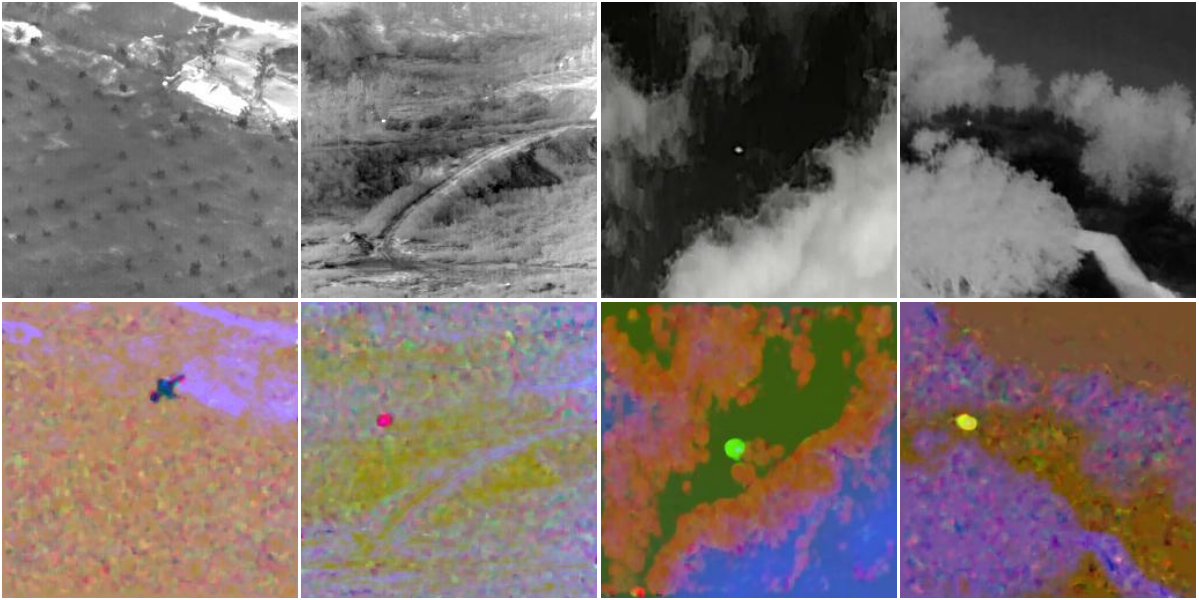}
  \vspace{-0mm} 
  \caption{Optical flow diagrams estimated for multiple infrared image scenes. The bottom double-channel color images represent the intensities of motion matrices $({\bf D}_x,{\bf D}_y)$.}
  \label{fig:optical_flow}
  \vspace{-0cm}
\end{figure}

\textbf{\textit{Step} 1.} (Calculate optical flow map) Given a pair of input images from adjacent frames, \({\bf I}_1, {\bf I}_2 \in \mathbb{R}^{n_1 \times n_2}\), we apply the Farneback method to calculate the motion vector for each pixel point. This results in the motion matrices \(\mathbf{D}_x \in \mathbb{R}^{n_1 \times n_2}\) and \(\mathbf{D}_y \in \mathbb{R}^{n_1 \times n_2}\), which represent the motion direction and intensity of all pixel points in the horizontal and vertical directions, respectively. 
The (absolute) optical flow map for each frame \(f = 1, 2, \dots, n_3\) is obtained by ${\bf M}_f=\sqrt{\left| {\bf D}_x \right|^2+\left| {\bf D}_y \right|^2}$ (element-wise operations).

\textbf{\textit{Step} 2.} (Dynamic multi-frame fusion) To enhance the robustness of motion estimation and mitigate transient outliers, we fuse the optical flow maps from both the current frame $f$ and its previous $k$ frames. Specifically, we calculate the maximum intensity value $m$ of the current frame's optical flow map ${\bf M}_f$, and use $m$ as the motion confidence to construct a dynamic multi-frame fusion model: $\mathbf{M}_{f}^{F}=\alpha \mathbf{M}_f+\left( 1-\alpha \right) \left( \sum_{i=1}^k\frac{1}{k}{\mathbf{M}_{f-i}} \right) ,\alpha =\frac{m}{m+\beta}$. Here, $\beta$ is a tuning parameter and is set to 0.1. {When the target's motion intensity $m$ in the current frame is large, $\alpha$ approaches 1, resulting in complete reliance on the current frame. Conversely, when the target motion intensity $m$ in the current frame is low, $\alpha$ approaches 0, leading to a greater reliance on historical frames.} We can obtain all the optical flow maps ${\bf M}_f^F\left( f=1,2,\dots ,n_3 \right)$ along the mode-3 direction, and by stacking, we get the final fused optical flow tensor $
    \mathcal{M} ^F={\rm stack}\left( {\bf M}_{1}^{F},...,{\bf M}_{f}^{F},...,{\bf M}_{n_3}^{F} \right)\in \mathbb{R} ^{n_1\times n_2\times n_3}$.

\textbf{\textit{Step} 3.} (Motion enhancement) The weight map $\mathcal{M} ^F$ exhibits higher magnitudes in regions with salient motion targets (see Figure \ref{fig:optical_flow}). By performing weighted sum with the original image $\mathcal{D}$, the response intensity of target regions will be enhanced, thereby improving detection accuracy. 
The motion enhancement model is formulated by $\mathcal{X} =\left( 1-\gamma \right) \mathcal{D} +\gamma \mathcal{M} ^F,$ where $\mathcal{D}$ is the original infrared tensor, $\mathcal{X}$ is the motion-enhanced tensor, and $\gamma$ is a balancing factor. The moving small targets in $\cal X$ are robustly strengthened, thereby enhancing detection accuracy.
\subsection{Nonlocal Grouping for STT}
The background of infrared images usually exhibits strong {\it nonlocal} low-rankness. 
To capture this property, we conduct patch tensor splitting on the motion-enhanced tensor $\mathcal{X} \in \mathbb{R}^{n_1 \times n_2 \times n_3}$ to obtain a series of small basic patch tensors $\mathcal{P} \in \mathbb{R} ^{p\times p\times n_3}$. Specifically, we slide a three-dimensional window over $\mathcal{X}$. The size of the window is $p\times p$, and the depth is $n_3$. To avoid redundancy, {we adopt a non-overlapping patch split strategy \cite{wang2021infrared, luo2022imnn}}, with the sliding step equal to $p$. Thus, we can obtain $L=\left( n_1/p \right) \left( n_2/p \right)$ basic patch tensors to compose a basic patch tensors set $\left\{ \mathcal{P}_l\in\mathbb{R} ^{p\times p\times n_3} \right\} _{l=1}^{L}$. We conduct the same splitting for a coarse background tensor $\mathcal{X}'$ to obtain coarse patches $\left\{ \mathcal{P}'_l\in\mathbb{R} ^{p\times p\times n_3} \right\} _{l=1}^{L}$, where $\mathcal{X}'$ is obtained by leveraging the low-rank tensor function representation (LRTFR) \cite{LRTFR}. This is achieved by optimizing the following model:
\begin{equation}\label{eq:LRTFR}
\begin{split}
&\min_{{\mathcal C},{\theta_1,\theta_2,\theta_3}}\sum_{(i,j,k)}\left|f_\Theta(i,j,k)-{\mathcal X}(i,j,k)\right|,\\&f_\Theta(i,j,k)={\mathcal C}\times_1f_{\theta_1}(i)\times_2f_{\theta_2}(j)\times_3f_{\theta_3}(k),
\end{split}
\end{equation}
where $f_\Theta(i,j,k)$ denotes the LRTFR \cite{LRTFR} that represents the tensor ${\mathcal X}$ with a Tucker tensor decomposition parameterized by INRs. The loss in \eqref{eq:LRTFR} is equivalent to using an $\ell_1$-norm loss conditioned on the infrared images ${\mathcal X}$ to optimize a low-rank tensor model, hence the learned tensor ${\mathcal X}'$ defined by ${\mathcal X}'(i,j,k)=f_\Theta(i,j,k)$ would contain the low-rank background information in the infrared images. We use this coarse background for nonlocal search.\par 
\begin{figure*}[!ht]
  \centering
  \includegraphics[width=0.95\linewidth]{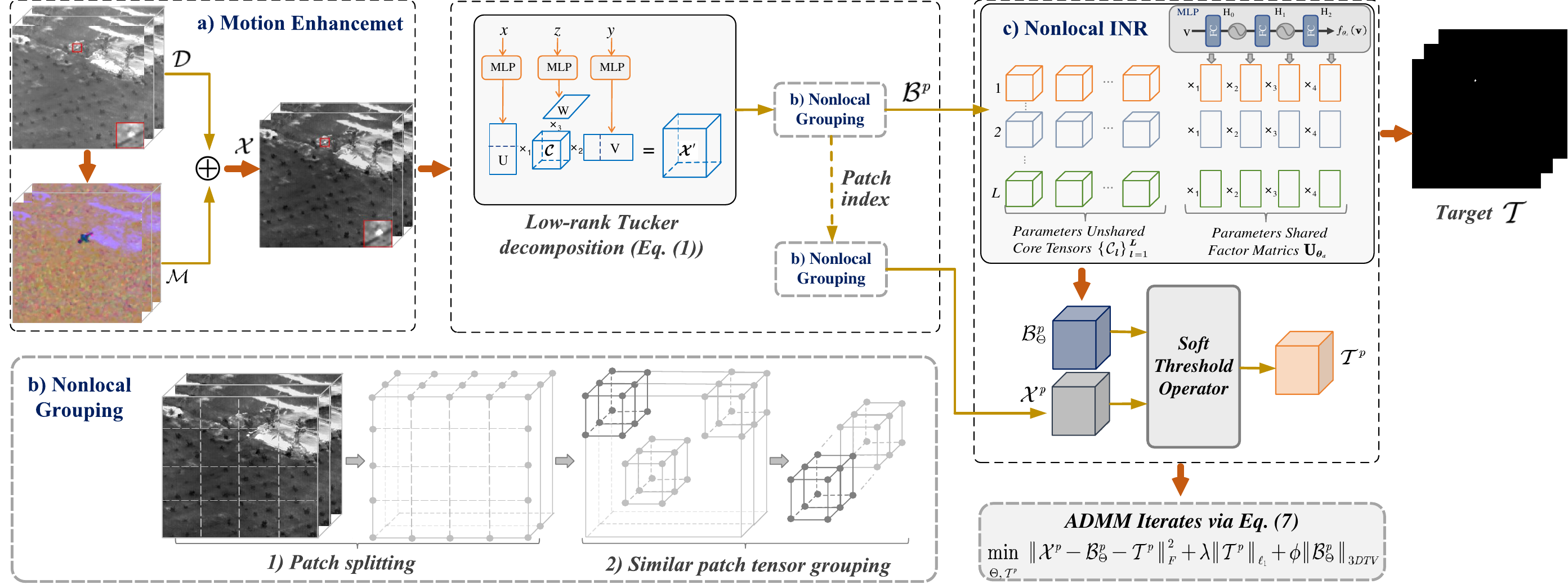}
  \vspace{-0mm} 
  \caption{Overall flowchart of our optNL-INR for IRSTD. a) Motion enhancement integrates the original image $\mathcal{D}$ with the optical flow map $\mathcal{M}$ to generate an enhanced image $\mathcal{X}$. b) Nonlocal grouping employs patch split and grouping to obtain nonlocal similar STT models $\mathcal{X}^p$ and $\mathcal{B}^p$. c) Nonlocal INR represents the background $\mathcal{B}^p$ in a continuous domain to obtain the representation $\mathcal{B}_{\Theta}^{p}$. The separated target patch tensor $\mathcal{T}^p$ is computed via ADMM and reconstructed into the target image $\mathcal{T}$.}
  \label{fig:Frame}
  \vspace{-0cm}
\end{figure*}
Specifically, we conduct nonlocal patch searching using the coarse background tensor ${\mathcal X}'$. For each non-overlapping patch tensor $\mathcal{P}' _{l}$, we aim to find its top-\textit{S} similar patch tensors in the basic patch tensors set $\left\{ \mathcal{P}'_l \right\} _{l=1}^{L}$. The distance between any two patch tensors is calculated by the Euclidean distance between them. Hence, for each patch $\mathcal{P}' _{l}$, we will find its top-$S$ similar patches in $\left\{ \mathcal{P}'_l \right\} _{l=1}^{L}$, and we denote the index set of all nonlocal similar patches of $\mathcal{P}' _{l}$ as $I_l$, which stores all the similar patch indexes of $\mathcal{P}' _{l}$. Then, for each motion-enhanced patch tensor $\mathcal{P}_{l}$ from $\cal X$, we concat it with other $S$ patch tensors in $\{\mathcal{P}_{l}\}_{l=1}^L$ by using the patch indexes $I_l$. These patches would share similar background information due to the nonlocal similarity. Such concat leads to a nonlocal tensor of size $p\times p\times n_3\times (S+1)$ for each $\mathcal{P}_{l}$.
We then concat all the nonlocal tensors to aggregate them into a five-dimensional tensor $\mathcal{X}^p\in \mathbb{R} ^{L\times p\times p\times n_3\times \left( S+1 \right)}$, where the $\left(l,:,:,:,: \right)$-th sub-tensor stores the $l$-th key patch $\mathcal{P}_{l}$ and its $S$ nonlocal similar patches (in terms of background similarity) across the spatial-temporal domain. This sub-tensor $\mathcal{X}^p(l,:,:,:,:)$ would enjoy strong low-rankness due to its composition of nonlocal similar patches across the infrared images. We illustrate the nonlocal searching algorithm in Algorithm \ref{alg:nonlocal}.
{\begin{algorithm}[t]
\caption{Nonlocal grouping for infrared images}
\label{alg:nonlocal}
\begin{algorithmic}[1]
\Require The motion-enhanced infrared images $\mathcal X$;
\Ensure The nonlocal similarity aggregated tensor ${\mathcal X}^p$;
\State Estimate coarse background ${\mathcal X}'$ using LRTFR \eqref{eq:LRTFR};
\State Split ${\mathcal X}$ and ${\mathcal X}'$ into non-overlapping patches $\{{\mathcal P}_l\in{\mathbb R}^{p\times p\times n_3}\}_{l=1}^L$ and $\{{\mathcal P}'_l\in{\mathbb R}^{p\times p\times n_3}\}_{l=1}^L$;
\State For each patch ${\mathcal P}'_l$, search the top-$S$ nonlocal similar patches in $\{{\mathcal P}'_l\}_{l=1}^L$ using Euclidean distance, denote $I_l$ the corresponding nonlocal patch index set;
\State Group each patch ${\mathcal P}_l$ with $S$ patches in $\{{\mathcal P}_l\}_{l=1}^L$ based on the patch index $I_l$ to form a nonlocal tensor $p\times p\times n_3\times (S+1)$;
\State Concat these nonlocal tensors to obtain the aggregated tensor ${\mathcal X}^p\in{\mathbb R}^{L\times p\times p\times n_3\times (S+1)}$;
\end{algorithmic}
\end{algorithm}}

\subsection{Nonlocal INR for Background STT Modeling}
We conduct the IRSTD fully based on the nonlocal patch dimension of ${\mathcal X}^p$. The RPCA framework for IRSTD can be re-formulated in terms of nonlocal patches modeling:
\begin{equation}\label{patch_RPCA_model_1}
\begin{aligned}
    \mathcal{X}^p=\mathcal{B}^p+\mathcal{T}^p +\mathcal{N}^p,
\end{aligned}    
\end{equation}
where $\mathcal{X} ^p,\mathcal{B} ^p,\mathcal{T} ^p,\mathcal{N} ^p\in \mathbb{R} ^{L\times p\times p\times n_3\times \left( S+1 \right)}$ represent the motion-enhanced, background, target, and noise patch tensors, respectively. By leveraging the nonlocal low-rankness of the background and the sparsity of targets, the IRSTD can be transformed into an optimization problem to minimize background rank and target sparsity. The noise term $\mathcal{N}^p$ can be transformed into an $F$-norm constraint. Therefore, the optimization problem is:
\begin{equation}\label{patch_RPCA_model_2}
\begin{aligned}
    \underset{\mathcal{B} ^p,\mathcal{T} ^p}{\min}\,\,\left\| \mathcal{X} ^p-\mathcal{B} ^p-\mathcal{T} ^p \right\| _F^2+rank\left( \mathcal{B} ^p \right) +\lambda \left\| \mathcal{T} ^p \right\|_{\ell_1},
\end{aligned}    
\end{equation}
where $\left\| \cdot \right\|_{\ell_1}$ denotes the $\ell_1$-norm to enforce sparsity. {To encode nonlocal low-rankness of background, we introduce a novel nonlocal INR method, which simultaneously preserves nonlocal low-rankness and captures spatial-temporal correlations of STT by INR smoothness \cite{LRTFR}.}\par 
Specifically, the nonlocal background patch tensor ${\mathcal B}^p$ holds strong low-rankness, hence we use the low-rank Tucker decomposition as introduced in Lemma \ref{Tucker} to model its nonlocal low-rankness. To better capture spatial-temporal correlations of STT, we propose to use INRs \cite{SIREN,LRTFR} to generate the factor matrices of the nonlocal Tucker decomposition model. The INRs hold inherent global Lipchitz smoothness \cite{LRTFR}, and hence could implicitly capture the spatial-temporal correlations of STT; see Theorem \ref{the_INR_smooth}. Formally, the nonlocal INR model for representing background STT ${\mathcal B}^p$ is defined as
\begin{equation}\label{NL-INR}
\begin{split}
&{\mathcal B}^p_\Theta(l,:,:,:,:):={\mathcal C}_l\times_1
{\bf U}_{\theta_1}\times_2
{\bf U}_{\theta_2}\times_3
{\bf U}_{\theta_3}\times_4
{\bf U}_{\theta_4},\;\\
&\qquad \qquad \qquad \quad (l=1,\cdots,L),\\
&{\bf U}_{\theta_d}(i_d,:)=f_{\theta_d}(i_d)\in{\mathbb R}^{r_d},\;(i_d=1,\cdots,n_d,\;d=1,2,3,4),
\end{split}
\end{equation}
where ${\mathcal B}^p_\Theta\in{\mathbb R}^{L\times n_1\times n_2\times n_3\times n_4}$ ($n_1=n_2=p,n_4=S+1$) denotes the learned background patch tensor parameterized by the nonlocal INR with parameters $\Theta:=\{\{{\mathcal C}_l\}_{l=1}^L,\theta_1,\theta_2,\theta_3,\theta_4\}$. The $\{{\mathcal C}_l\}_{l=1}^L$ are $L$ unshared core tensors for the $L$ nonlocal groups. The ${\bf U}_{\theta_d}\in{\mathbb R}^{n_d\times r_d}$ is the shared factor matrix across all nonlocal groups $l=1,\cdots,L$. The shared factor matrix ${\bf U}_{\theta_d}$ is generated by the INR $f_{\theta_d}(\cdot):{\mathbb R}\rightarrow{\mathbb R}^{r_d}$. Such an INR is an MLP with sine activation functions \cite{SIREN}, which maps a tensor index $i_d$ to the corresponding factor vector ${\bf U}_{\theta_d}(i_d,:)$: 
\begin{equation}\label{U_model}
\begin{aligned}
{\bf U}_{\theta_d}(i_d,:)=f_{\theta_{d}}(i_d) = {\bf H}_M(\sigma({\bf H}_{M-1} \cdots \sigma({\bf H}_{1}i_d))),
\end{aligned}
\end{equation}
where $\theta_{d}\triangleq\{{\bf H}_m\}_{m=1}^M$ are learnable weights of the INR and $\sigma(\cdot)\triangleq\sin(\omega\cdot)$ denotes the sinusoidal activation function with a frequency parameter $\omega$ \cite{SIREN}. The strong continuous representation ability of INR makes it effective for generating the nonlocal Tucker model.\par 
The nonlocal INR implicitly exploits the nonlocal low-rankness underlying STT through the tensor decomposition, and hence serves as an effective rank regularization $rank({\mathcal B}^p)$ in \eqref{patch_RPCA_model_2}. The implicit regularization brought by INRs further improves spatial-temporal correlation excavation. Based on the nonlocal INR modeling \eqref{NL-INR}, the proposed IRSTD optimization model is formulated as
\begin{equation}\label{nonlocal_INR_model}
\begin{aligned}
    \underset{\Theta,\mathcal{T} ^p}{\min}\,\,\left\| \mathcal{X} ^p-\mathcal{B} ^p_\Theta-\mathcal{T} ^p \right\| _F^2+\lambda \left\| \mathcal{T} ^p \right\|_{\ell_1},
\end{aligned}    
\end{equation}
where the optimization parameters include the nonloal INR parameters $\Theta$ and the sparse target $\mathcal{T} ^p$.
\subsection{Theoretical Validation for Nonlocal INR}
We provide rigorous theoretical analysis for optNL-INR, including: (1) the existence of the nonlocal INR model that guarantees the representation ability of the model to fully capture the background for IRSTD (Theorem \ref{the_INR}); (2) the spatial-temporal correlation bound that reveals the spatial-temporal smoothness constraint brought by the INR model (Theorem \ref{the_INR_smooth}); and (3) the ADMM fixed-point convergence that guarantees the numerical stability of the proposed IRSTD algorithm (Lemma \ref{lemma_convergence}).
\begin{theorem}[Existence of the nonlocal INR model]\label{the_INR}
Suppose that the factor INR $f_{\theta_d}(\cdot)$ is an universal approximator over any functions in $\{f:{\mathbb R}\rightarrow{\mathbb R}^{r_d}\}$ ($d=1,2,3,4$), then provided that the background nonlocal tensor ${\mathcal B}^p$ is of low-rank structures ${\rm rank}_T(\mathcal{B}^p(l,:,:,:,:))\leq(r_1,r_2,r_3,r_4)$ for any $l$, then there must exist a group of parameters $\Theta:=\{\{{\mathcal C}_l\}_{l=1}^L,\theta_1,\theta_2,\theta_3,\theta_4\}$ that satisfy $\mathcal{B}^p(l,i_1,i_2,i_3,i_4)={\mathcal C}_l\times_1f_{\theta_1}(i_1)\times_2f_{\theta_2}(i_2)\times_3f_{\theta_3}(i_3)\times_4f_{\theta_4}(i_4)$.
\end{theorem}
\begin{proof}
The conjecture follows from the combination of: (1) the existence property of the low-rank tensor function factorization in Theorem 2 of \cite{LRTFR}, by viewing each $\mathcal{B}^p_\Theta(l,:,:,:,:)$ as a discrete tensor sampled on a low-rank tensor function with function rank less than $(r_1,r_2,r_3,r_4)$, and (2) the universal approximation of the INR (i.e., MLP) for representing any functions on the Euclidean space. 
\end{proof}
The existence result justifies the rationality of the nonlocal INR for background modeling in infrared images, showing that the nonlocal INR could fully represent the low-rank background. Moreover, the nonlocal INR implicitly captures the global spatial-temporal correlations in infrared images by the continuous function representation of INRs, thereby enhancing the performance for IRSTD. We establish the theoretical implicit regularization of our model as follows.
\begin{theorem}\label{the_INR_smooth}
Let $f_{\theta_1}(\cdot)$, $f_{\theta_2}(\cdot)$,$f_{\theta_3}(\cdot)$,$f_{\theta_{4}}(\cdot)$ be factor INRs with sine activation function $\sin(\omega \;\cdot)$ and depth $M$. Assume that the $\ell_1$-norm of each core tensor ${\mathcal C}_l$ ($l=1,2,\cdots,L$) and each weight matrix of factor INRs is bounded by $\eta$. Denote the background patch tensor generated by such nonlocal INRs as $\mathcal{B}^p(l,i_1,i_2,i_3,i_4)={\mathcal C}_l\times_1f_{\theta_1}(i_1)\times_2f_{\theta_2}(i_2)\times_3f_{\theta_3}(i_3)\times_4f_{\theta_4}(i_4)$. Then, we have the following smoothness bound that reveals the global spatial-temporal correlation:
\begin{equation*}\label{eq_lip}
			\begin{split}
|\mathcal{B}^p_\Theta(l_1,i_1,\cdots,i_d,\cdots,i_4)-\mathcal{B}^p_\Theta(l_2,i_1,&\cdots,i_d',\cdots,i_4)|\\&\leq\delta_1 |i_d-i_d'| + \delta_2,
			\end{split}
		\end{equation*} 
		where $\delta_1,\delta_2$ are constants.
		The inequality holds for any two nonlocal groups $l_1,l_2$ $\in\{1,2,\cdots,L\}$, dimension $d\in\{1,2,3,4\}$, and tensor indexes $i_d,i_d'$. When $l_1=l_2$, the upper bound reduces to ${\delta_1 |i_d-i_d'|}$ (i.e., $\delta_2$ vanishes).
\end{theorem}
The proof of Theorem \ref{the_INR_smooth} is placed in supplementary file. When $d=1,2$, the smoothness bound reveals the spatial smoothness embedded in the model, and when $d=3$, the bound reveals the temporal smoothness. Hence, the global Lipschitz smoothness reveals the spatial-temporal correlations captured by INRs through bounding the differences between spatial-temporal elements of ${\mathcal B}_\Theta^p$ with some constants. This bound is intrinsically related to the Lipchitz continuous nature of INR \cite{LRTFR}. Especially, the nonlocal INR model differentiates between intra- and inter-relationships within nonlocal groups $\{{\mathcal{B}^p_\Theta(l,:,:,:)}\}_{l=1}^L$, i.e., the intra-relationship inside a group (when $l_1=l_2$, $\delta_2$ vanishes) is more pronounced than the inter-relationships between different nonlocal groups (when $l_1\neq l_2$). This theoretical result interprets the ability of nonlocal INRs to simultaneously characterize both nonlocal low-rankness and spatial-temporal correlation of infrared backgrounds in a flexible way, with such characterizations adaptively learned via a parametric nonlocal INR model~\eqref{NL-INR}.

\subsection{3DTV and ADMM Algorithm for IRSTD}
To more faithfully capture spatial-temporal local correlations of the background, we further introduce a spatial-temporal 3DTV regularization into the model:
\begin{equation}\label{ST-3DTV}
\begin{aligned}
    \left\| \mathcal{B}^p_\Theta \right\|_{\rm 3DTV}=\left\| \nabla _x\mathcal{B} ^p_\Theta \right\| _{\ell_1}+\left\| \nabla _y\mathcal{B} ^p_\Theta \right\| _{\ell_1}+\eta \left\| \nabla _z\mathcal{B} ^p_\Theta \right\| _{\ell_1},
\end{aligned}   
\end{equation}
where $\nabla_x,\nabla_y,\nabla_z$ are first-order difference operators and $\eta$ is a weight parameter. The final optimization problem in \eqref{nonlocal_INR_model} of our optNL-INR can be further expressed as:
\begin{equation}\label{patch_RPCA_model_3}
\begin{aligned}
    \underset{\Theta,\mathcal{T} ^p}{\min}\,\,\left\| \mathcal{X} ^p-\mathcal{B} ^p_\Theta-\mathcal{T} ^p \right\| _F^2+\lambda \left\| \mathcal{T} ^p \right\|_{\ell_1}+\phi\left\| \mathcal{B}^p_\Theta \right\|_{\rm 3DTV},
\end{aligned}   
\end{equation}
where $\phi$ is a trade-off parameter. \par
To address the optimization model \eqref{patch_RPCA_model_3}, we employ an auxiliary variable $\mathcal A$ and a Lagrange multiplier $\Lambda$, and transform the global optimization into several subproblems under the ADMM framework:
\begin{equation}\label{ADMM}
	\begin{aligned}
	&\min_{{\mathcal A}}\|{\mathcal X}^p-{\mathcal A}-{\mathcal T}_{t}^p\|_F^2+\frac{\rho_t}{2}\|{\mathcal A}-{\mathcal B}_{\Theta_{t}}^p+{\Lambda}_{t}\|_F^2,\\
	&\min_{\Theta}\frac{\rho_t}{2}\|{\mathcal A}_{t+1}-{\mathcal B}_\Theta^p+\Lambda_{t}\|_F^2+\phi\|{\mathcal B}_\Theta^p\|_{\rm 3DTV},\\
	&\min_{{\mathcal T}^p}\|{\mathcal X}^p-{\mathcal A}_{t+1}-{\mathcal T}^p\|_F^2+\lambda\|{\mathcal T}^p\|_{\ell_1},\\
	&\;{\Lambda}_{t+1}={\Lambda}_t+{\mathcal A}_{t+1}-{\mathcal B}_{\Theta_{t+1}}^p,\;\rho_{t+1}=\kappa\rho_{t},
	\end{aligned}
\end{equation}
where $\kappa>1$ is a constant in the ADMM framework, $t$ denotes the index of iteration number, and $\rho_t$ is a trade-off parameter that evolves with iterations. Here, the ${\mathcal A}$-subproblem consists of squared terms and can be solved by first-order optimal condition. The $\Theta$-subproblem is optimized through utilizing the Adam \cite{Adam} optimizer in each iteration of the ADMM, which is viewed as a plug-and-play denoising subproblem under the ADMM framework. The sparse target ${\mathcal T}^p$-subproblem is employed to accurately separate and extract the target from the background. It can be solved efficiently by the soft threshold shrinkage operator \cite{donoho2002noising}:
\begin{equation}\label{soft_threshold}
\begin{aligned}
    \mathcal{T} ^p_{t+1}=\mathcal{S} _{\frac{\lambda}{2}}\left( \mathcal{X} ^p-\mathcal{A}_{t+1}\right),
\end{aligned}   
\end{equation}
where $\mathcal{S} \left( \cdot \right)$ represents the soft-thresholding shrinkage operator defined as $
\mathcal{S} _{\xi}\left( x \right) =\mathrm{sign}\left( x \right) \cdot \max \left( \left| x \right|-\xi ,0 \right)$. After the optimization of ADMM, we reshape the optimized sparse target ${\mathcal T}^p\in{\mathbb R}^{L\times p\times p\times n_3\times (S+1)}$ to the original STT shape $n_1\times n_2\times n_3$ to obtain the IRSTD result. \par
The computational complexity of our ADMM is $O(4MW^3+Lr)$ at each iteration, where $M$ and $W$ denote the network depth and width, $L$ denotes the number of nonlocal patches, and $r$ denotes the rank.\par 
Under mild assumptions of the bounded denoiser \cite{PnP}, we have the following fixed-point convergence guarantee for the plug-and-play ADMM algorithm (Lemma \ref{lemma_convergence}). Experimental results using relative error (RE) curves on three infrared scenes (as shown in Figure \ref{fig:converge}) demonstrate the numerical convergence behavior of our algorithm.
\begin{lemma}\label{lemma_convergence}
Assume that the $\Theta$-subproblem in \eqref{ADMM} is bounded $\|{\mathcal B}_{\Theta_{t+1}}^p-({\mathcal A}_{t+1}+{\Lambda}_t)\|_{F}^2\leq{\frac{C}{\rho^t}}$ for a constant $C$. Then the iterates in \eqref{ADMM} admit a fixed-point convergence, i.e., there exist $({\mathcal A}^*,\Theta^*,{\Lambda}^*)$ such that $\|{\mathcal A}_t-{\mathcal A}^*\|_F^2\rightarrow 0$, $\|{\mathcal B}_{\Theta_t}^p-{\mathcal B}_{\Theta^*}^p\|_F^2\rightarrow 0$, and $\|{\Lambda}_t-{\Lambda}^*\|_F^2\rightarrow 0$ as $t\rightarrow \infty$.
\end{lemma}
\begin{table*}[!ht]
\setlength{\tabcolsep}{5pt}
\vspace{0pt}  
\centering   
    \belowrulesep=0.2pt
    \aboverulesep=0.2pt
\caption{Average results on ATR and Anti-UAV datasets. The inference times (s) and number of parameters (M) are reported. \\The best results are \textbf{bolded}, and the second-best are \underline{underlined}.}  
\label{tab:performance_average}     
\begin{tabular}{c|c|c|cccc|cccc|c|c}
    \toprule
    \multirow{2}{*}{\textbf{Scheme}} & \multirow{2}{*}{\textbf{Method}} & \multirow{2}{*}{\textbf{Category}} & \multicolumn{4}{c}{\textbf{(A) ATR dataset}} & \multicolumn{4}{c|}{\textbf{(B) Anti-UAV dataset}}   & \multirow{2}{*}{\textbf{\shortstack{Inference \\ Time (s)}}} & \multirow{2}{*}{\textbf{\shortstack{Params. \\(M)}}}\\
    \cmidrule(lr){4-7} \cmidrule(lr){8-11} 
    & & & $IoU$ & $F_1$ & $P_d$ & $F_a \downarrow$ & $IoU$ & $F_1$ & $P_d$ & $F_a \downarrow$ & & \\
    \midrule
    \multirow{9}{*}{Supervised} 
    &	ACM	\cite{dai2021asymmetric}	&	single-DL	&	14.34 	&	19.88 	&	39.28 	&	8.07 	&	25.49 	&	39.29 	&	54.75 	&	7.70 	&	0.113 	&	0.40 	\\
    &	AGPCNet	\cite{zhang2023attention}	&	single-DL	&	17.80 	&	25.29 	&	40.72 	&	10.28 	&	31.01 	&	45.72 	&	61.75 	&	3.18 	&	0.098 	&	12.36 	\\
    &	RDIAN	\cite{sun2023receptive}	&	single-DL	&	22.25 	&	28.52 	&	48.91 	&	2.32 	&	16.07 	&	24.72 	&	40.13 	&	4.29 	&	0.097 	&	0.22 	\\
    &	PBT	\cite{yang2024pbt}	&	single-DL	&	31.04 	&	40.58 	&	69.69 	&	6.57 	&	28.88 	&	38.58 	&	67.00 	&	83.06 	&	0.132 	&	26.54 	\\
    &	RPCANet	\cite{wu2024rpcanet}	&	single-DL	&	24.78 	&	33.16 	&	53.16 	&	2.26 	&	20.55 	&	30.45 	&	41.50 	&	7.82 	&	0.112 	&	0.68 	\\
    &	SCTransNet	\cite{yuan2024sctransnet}	&	single-DL	&	40.29 	&	51.04 	&	58.75 	&	3.38 	&	38.60 	&	50.46 	&	57.93 	&	5.52 	&	0.093 	&	11.19 	\\
    &	APTNet	\cite{zhang2025aptnet}	&	single-DL	&	46.46 	&	56.35 	&	83.59 	&	4.49 	&	54.65 	&	70.06 	&	84.40 	&	3.06 	&	0.164 	&	5.60 	\\
    &	ILNet	\cite{li2025ilnet}	&	single-DL	&	52.16 	&	62.31 	&	82.66 	&	\underline{1.21}	&	54.76 	&	67.87 	&	78.75 	&	1.91 	&	0.153 	&	4.04 	\\
    &	LMAFormer	\cite{huang2024lmaformer}	&	multi-DL	&	53.70 	&	65.60 	&	84.84 	&	2.30 	&	55.03 	&	65.83 	&	\underline{88.78}	&	0.65 	&	0.382 	&	390.05 	\\
    \midrule 
    \multirow{13}{*}{Unsupervised} 
    &	NTFRA	\cite{kong2021infrared} 	&	single-Opti	&	12.85 	&	20.06 	&	51.72 	&	126.0 	&	15.83 	&	25.62 	&	52.88 	&	13.33 	&	1.692 	&	-	\\
    &	SRWS	\cite{zhang2021infrared} 	&	single-Opti	&	23.95 	&	34.65 	&	56.09 	&	6.86 	&	22.84 	&	36.83 	&	66.25 	&	5.15	&	0.832 	&	-	\\
    &	HiLV-LRSD	\cite{liu2023single}	&	single-Opti	&	26.90 	&	38.54 	&	66.09 	&	10.07 	&	11.85 	&	17.61 	&	52.00 	&	135.3 	&	0.266 	&	-	\\
    &	IMNN-LWEC	\cite{luo2022imnn} 	&	multi-Opti	&	30.46 	&	43.57 	&	53.28 	&	2.59 	&	25.02 	&	37.32 	&	55.27 	&	1.84 	&	2.299 	&	-	\\
    &	MFSTPT	\cite{hu2022infrared}  	&	multi-Opti	&	17.74 	&	28.80 	&	73.03 	&	10.77 	&	23.07 	&	35.52 	&	67.43 	&	4.24 	&	19.656 	&	-	\\
    &	SRSTT	\cite{li2023SRSTT}	&	multi-Opti	&	40.89 	&	53.09 	&	64.00 	&	6.40 	&	40.73 	&	50.69 	&	61.38 	&	4.58 	&	12.815 	&	-	\\
    &	STRL-LBCM	\cite{luo2023strl}	&	multi-Opti	&	38.08 	&	48.71 	&	61.56 	&	1.52 	&	41.50 	&	52.53 	&	45.83 	&	6.93 	&	0.924 	&	-	\\
    &	WASpN-STTTV	\cite{ma2023weighted}	&	multi-Opti	&	53.17 	&	68.20 	&	78.91 	&	2.15 	&	53.60 	&	67.63 	&	78.25 	&	1.50 	&	3.145 	&	-	\\
    &	NFTDGSTV	\cite{liu2023infrared} 	&	multi-Opti	&	42.87 	&	56.24 	&	75.94 	&	5.17 	&	40.42 	&	54.41 	&	77.76 	&	1.65 	&	1.858 	&	-	\\
    &	TCTV	\cite{liu2024infrared}	&	multi-Opti	&	53.84 	&	67.47 	&	82.50 	&	2.62 	&	47.19 	&	62.80 	&	72.88 	&	4.81 	&	2.592 	&	-	\\
    &	CTSTC	\cite{luo2024clustering} 	&	multi-Opti	&	53.83 	&	65.14 	&	82.50 	&	8.16 	&	52.00 	&	67.88 	&	81.00 	&	1.52 	&	0.746 	&	-	\\
    &	3DSTPM	\cite{zhang3DSTPM}	&	multi-DL	&	54.92 	&	68.16 	&	80.53 	&	2.01 	&	49.92 	&	61.91 	&	76.83 	&	2.08 	&	1.465 	&	5.27 	\\
    &	NeurSTT	\cite{wu2025neural}	&	multi-DL	&	\underline{62.25}	&	\underline{71.91}	&	\underline{87.88}	&	1.94 	&	\underline{61.97}	&	\underline{75.53}	&	88.25 	&	\underline{0.51}	&	0.674 	&	0.32 	\\
    \cmidrule{2-13}
    &	optNL-INR(Ours)		&	multi-DL	&	\textbf{68.77}	&	\textbf{80.97}	&	\textbf{92.81}	&	\textbf{0.55}	&	\textbf{69.74}	&	\textbf{81.88}	&	\textbf{92.78}	&	\textbf{0.13}	&	0.635 	&	0.35 	\\
    \bottomrule
\end{tabular}   
\end{table*}

\begin{figure}[t]
    \centering
    \includegraphics[
        width=0.72\linewidth,height=0.4\linewidth 
    ]{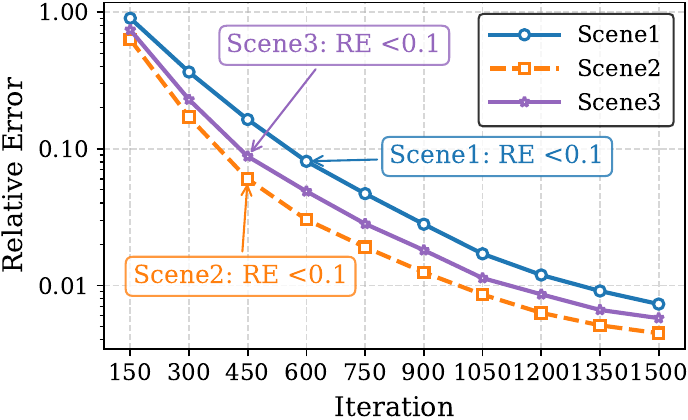}     
    \vspace{0cm}
    \caption{Convergence curves using relative error between $\mathcal{T}_{t-1}^p$ and $\mathcal{T}_t^p$ in the proposed optNL-INR algorithm.}
    \label{fig:converge}
    \vspace{0cm} 
\end{figure}
\begin{figure}[ht]
    \centering
    \includegraphics[
        width=0.9\linewidth
    ]{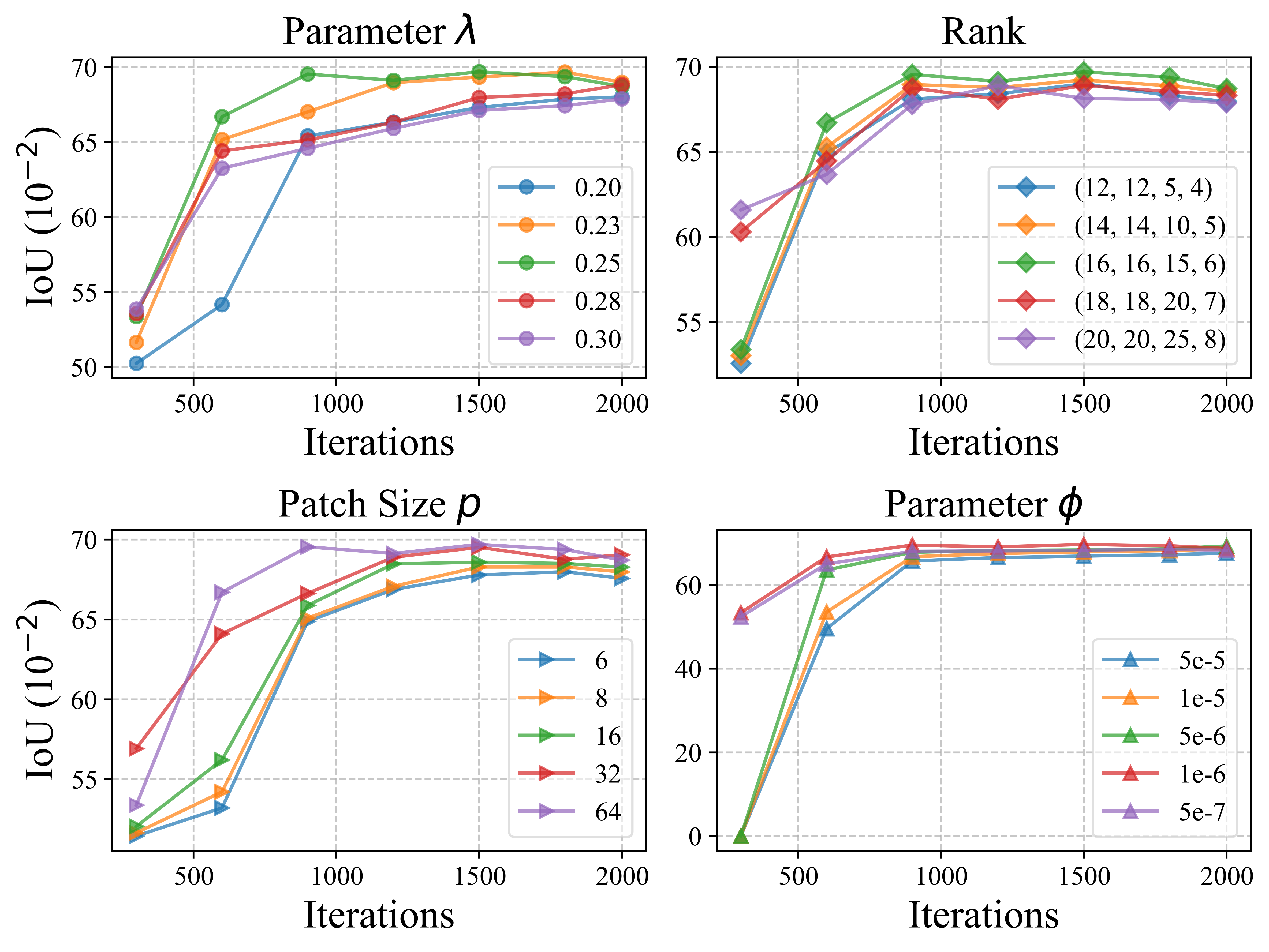}       
    \vspace{-0.4cm}
    \caption{$IoU$ curves w.r.t. hyperparameters in optNL-INR.}
    \label{fig:params}
    \vspace{-0.6cm} 
\end{figure}
\begin{figure}[t]
    \centering
    \includegraphics[
        width=\linewidth, 
    ]{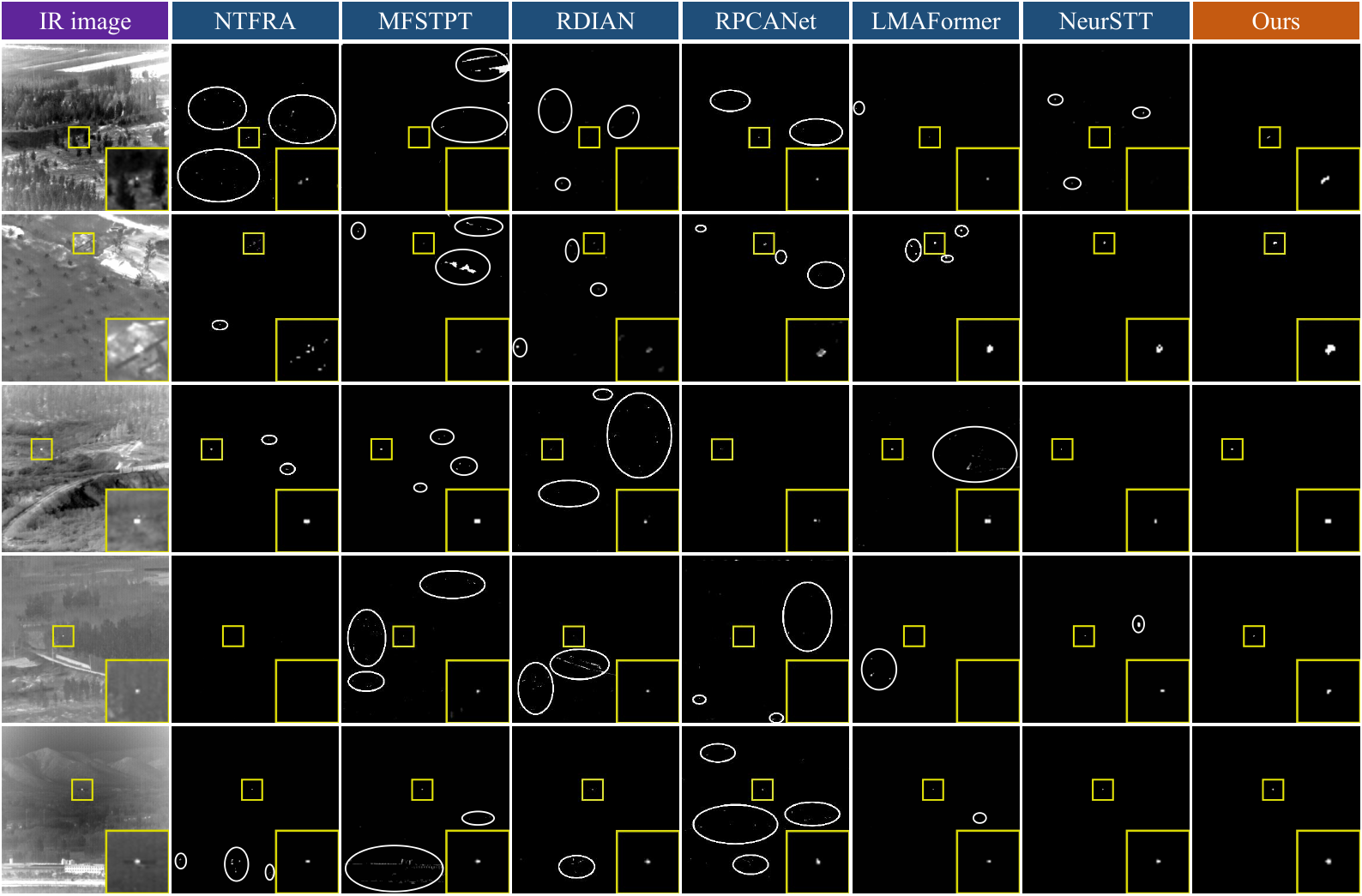}\vspace{-0.1cm}       
    \caption{Visual results for scenes A1-A3, B1-B3. Target region (yellow box) and false alarm (white circle) are marked.}
    \label{fig:visual_results}
    \vspace{-0.5cm} 
\end{figure}

\section{Experiments}
{\bf Settings} We conduct experiments on two public datasets for time-continuous IRSTD: (A) ATR\footnote{\url{http://www.sciencedb.cn/dataSet/handle/902}} (8 scenes) and (B) Anti-UAV\footnote{\url{https://codalab.lisn.upsaclay.fr/competitions/21688}} (10 scenes), with details in supplementary file. We use four evaluation metrics: pixel-level intersection over union ($IoU(10^{-2})$, with Tr=40\%), F-measure ($F_1(10^{-2})$), target-level detection probability ($P_d(10^{-2})$), and false alarm rate ($F_a(10^{-5})$). Detailed model implementation details are comprehensively provided in supplementary file Table 1.\\
{\bf{Convergence Analyses}} The relative error curves (the $F$-norm error between the target result $\mathcal{T}_t^p$ in the current iteration and $\mathcal{T}_{t-1}^p$ in the previous iteration) of our algorithm for three scenes are shown in Fig. \ref{fig:converge}. The RE converges to zero stably, validating the convergence behavior of our method.\\
{\bf Parameter {Sensitivity} Analyses} We test critical hyperparameters in optNL-INR and analyze the variation of average $IoU$, as shown in Fig. \ref{fig:params}. Our method is relatively robust to these parameters, including the trade-off parameters $\lambda,\phi$, the Tucker rank $(r_1,r_2,r_3,r_4)$, and the patch size $p$. More parameter analyses are provided in supplementary file.

\subsection{Comparison to State-of-the-Art Methods}
We compare the proposed optNL-INR with 9 supervised and 13 unsupervised state-of-the-art IRSTD methods to demonstrate the effectiveness (please see Table \ref{tab:performance_average} for detailed comparison methods and corresponding citations).
Hyperparameter settings for all methods are comprehensively provided in the supplementary file. We follow the hyperparameters in the source code and fine-tune them for optimal results.

{\bf Quantitative Results} The average quantitative results on two datasets are shown in Table \ref{tab:performance_average} (inference time is reported per frame). The proposed unsupervised optNL-INR method consistently delivers superior average results for two datasets across different evaluation metrics. Especially, our method achieves the best $IoU$ of 68.77\% and 69.74\% for ATR and Anti-UAV, outperforming the second-best by 6.52\% and 7.77\%. Multi-frame optimization methods generally outperform single-frame ones, highlighting the effectiveness of temporal information in enhancing detection performance. Single-frame deep learning methods show limited performance for the spatial-temporal data in the multi-frame datasets. This suggests that limited training sample diversity in sequential images hinders feature learning, even for supervised methods. In contrast, the multi-frame deep method NeurSTT yields impressive results in several scenes, while our proposed method delivers superior and more robust results. Overall, the proposed unsupervised optNL-INR demonstrates strong robustness by enhancing detection accuracy. While supervised deep learning methods offer faster inference via pre-trained models, our unsupervised method still remains competitive, processing at around 0.6s per frame compared to other optimization-based baselines.
\\{\bf Visual Results} Fig. \ref{fig:visual_results} presents six typical scenes for representative methods (from top to bottom present A1-A3 and B1-B3). Our method shows superior detection accuracy and robustness. For instance, RDIAN fails to detect targets in A1, and B1, while MFSTPT and NeurSTT also exhibit target loss in A1. MFSTPT generates predictions with background clutter in A1, A2, and B3, while RDIAN and LMAFormer show weak noise suppression in A3. In contrast, our method performs robustly in both target detection and noise suppression, and provides more accurate target shape predictions. In B1 and B2, MFSTPT produces fragmented target shapes, while LMAFormer’s predictions still deviate from the actual contours. Our method captures finer target details, as seen in B1 and B2, where the predicted shape is more precise. Additional visual results are provided in supplementary file.

\subsection{Ablation Study}
We conduct ablation for two key modules in our optNL-INR: the motion enhancement and the nonlocal INR. When nonlocal INR is disabled, we set the nonlocal patch size to the full image size, and our method reduces to a pure global low-rank-based approach. The results are shown in Table \ref{tab:component_Ablation}. Both components contribute to notable improvements. The nonlocal INR increases the $IoU$ from 66.42\% to 68.50\% and $F_1$ from 78.51\% to 80.16\%. This indicates that the nonlocal similarity improves background consistency and suppresses noise, improving the saliency of small targets. The motion enhancement module further increases $P_d$ from 88.14\% to 92.75\%, demonstrating that the motion enhancement using optical flow effectively mitigates moving target loss. In Fig. \ref{fig:optical_fused}, we demonstrate the effectiveness of the proposed dynamic multi-frame fusion strategy in the motion enhancement, where the motion confidence weight $\alpha$ is calculated dynamically for each frame. Compared to using no fusion or fixed fusion strategy (fix $\alpha=0.2$), the dynamic fusion approach robustly tracks temporal motion changes, leading to enhanced motion estimation and improved results.
\begin{table}[t]
\renewcommand{\arraystretch}{0.8}  
\setlength{\tabcolsep}{6pt}   
    \vspace{0pt}  
    \centering  
    \caption{Ablation results for key components in our proposed optNL-INR (motion enhancement and nonlocal INR).}
    \label{tab:component_Ablation}
    \vspace{0pt}  
    \begin{tabular}{cccccc}
        \toprule
        \multicolumn{2}{c}{\textbf{Module}} & \multicolumn{4}{c}{\textbf{Average Metrics}} \\
        \cmidrule(lr){1-2} \cmidrule(lr){3-6} 
        Motion & Nonlocal & $IoU$ & $F_1$ & $P_d$ & $F_a \downarrow$ \\
        \midrule
        ×	&	×	&	66.42 	&	78.51 	&	87.27 	&	\underline{0.45} 	\\
        ×	&	\checkmark	& \underline{68.50}	& \underline{80.16}	& \underline{88.14}	& 0.49	\\
        \checkmark	&	\checkmark	& \textbf{69.03} 	& \textbf{81.38} 	& \textbf{92.75} 	& \textbf{0.28} 	\\
        \bottomrule
    \end{tabular}      
\end{table}
\begin{figure}[t]
    \centering
    \includegraphics[
        width=0.92\linewidth
    ]{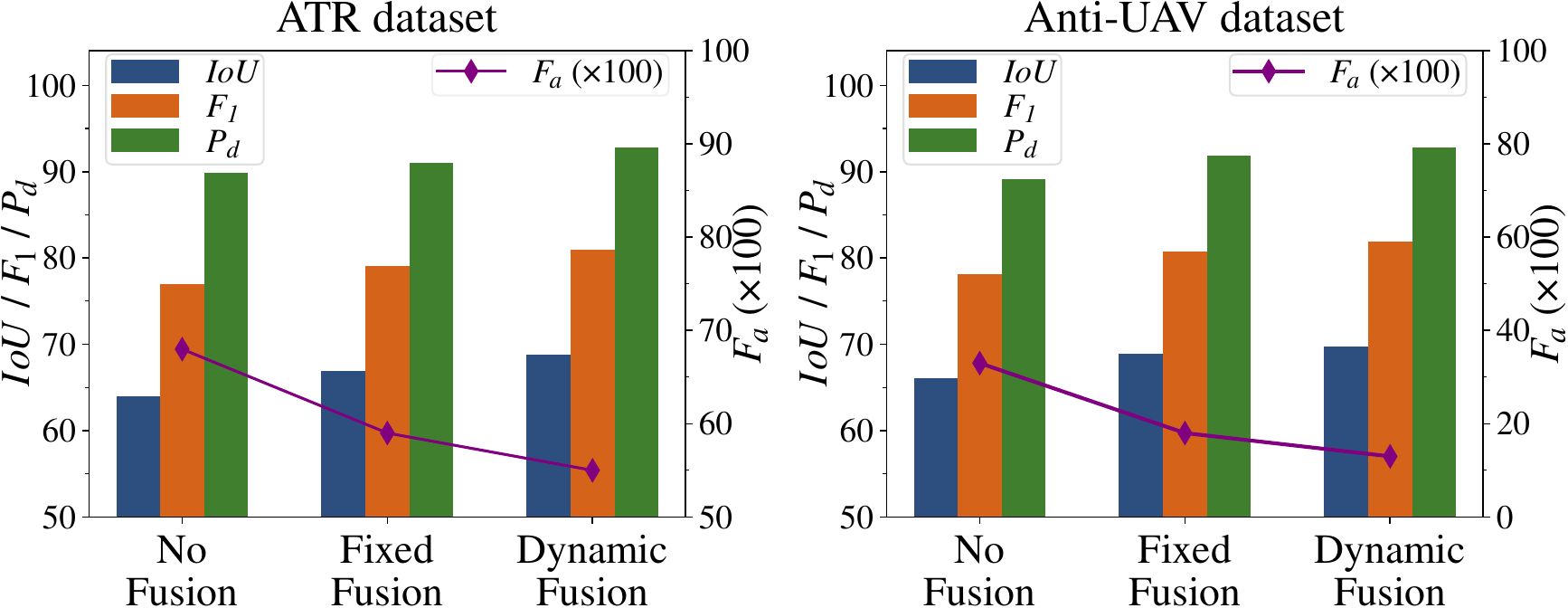}\vspace{-5pt}       
    \caption{Ablation results for the dynamic multi-frame fusion strategy in the motion estimation.}
    \label{fig:optical_fused}
    \vspace{-0.5cm} 
\end{figure}

\section{Conclusion}
We proposed an unsupervised motion-enhanced nonlocal similarity implicit neural representation model, termed optNL-INR, for infrared dynamic background estimation and small moving targets detection. The motion estimation better detects small targets, and the dynamic multi-frame fusion strategy improves robustness of motion saliency estimation. The nonlocal INR model effectively captures both the nonlocal low-rankness and the spatial-temporal correlations of the background tensor, thus achieving more accurate target-background separation. Extensive theoretical and experimental analyses have demonstrated the effectiveness and superiority of the proposed optNL-INR method for IRSTD.


\bibliographystyle{ieeetr}  
\bibliography{optNL-INR}

\end{document}